\documentclass[11pt]{article}
\usepackage{amsmath, amssymb, amsthm}
\usepackage{graphicx}
\usepackage{authblk}
\usepackage{times}
\usepackage{hyperref}
\usepackage{natbib}

\theoremstyle{plain}
\newtheorem{theorem}{Theorem}

\theoremstyle{definition}

\theoremstyle{remark}

\title{\textbf{Data Curation Through the Lens of Spectral Dynamics: Static Limits, Dynamic Acceleration, and Practical Oracles}}

\author[1]{Yizhou Zhang\footnote{Corresponding Author}}
\affil[1]{zyizhou96@gmail.com}
\author[2]{Lun Du}
\affil[2]{dulun2834@126.com}
\date{}

\begin{document}
\maketitle

\begin{abstract}
Large-scale neural models are increasingly trained with data pruning, synthetic data
generation, cross-model distillation, reinforcement learning from human feedback (RLHF), 
and difficulty-based sampling. While several of these data-centric strategies reliably 
improve training efficiency and downstream performance, others fail to provide meaningful 
gains---most notably self-generated synthetic data, which often increases dataset volume 
without enhancing model capability. This raises several fundamental questions: 
\emph{When does pruning help? Why can a small reward model improve a much larger language 
model? Why does a model's own synthetic data rarely make it smarter?}

We address these questions through a unified theoretical framework built upon 
\emph{generalized spectral dynamics}, building on recent advances in neural scaling laws 
and spectral analyses of representation learning 

We formalize data pruning as reweighting the sampling distribution and map its effect onto 
the eigenstructure of the data-induced operator. Our first main result shows that 
\textbf{static pruning induces a bounded operator and therefore cannot change the spectral 
tail exponent}; it provides at most finite-region improvements and cannot alter asymptotic 
neural scaling. Our second result analyzes \textbf{time-dependent 
pruning}, showing that an ideal oracle capable of tracking spectral residuals and 
continuously re-normalizing the tail can provably accelerate learning---although practical 
systems can only approximate this behavior.

We further identify \textbf{four recurring paradigms} in modern training pipelines ---teacher 
acceleration, heterogeneous-model complementarity, self-scoring, and self-generated synthetic 
data---and show how each corresponds to a partial approximation of the oracle. This explains 
why cross-model distillation and RLHF with small reward models often improve large networks 
\cite{hernandez2021scaling,wei2022emergent}, while self-generated synthetic data cannot 
expand capability because it remains confined to the model's existing spectral span.

Overall, this framework provides a principled operator-theoretic explanation for the 
empirical successes and limitations of pruning, synthetic data, and teacher--student 
interactions, clarifying both the \emph{possibilities} and the \emph{fundamental limits} 
of data-centric interventions for accelerating large-model training.
\end{abstract}

\section{Introduction}\label{sec:intro}

The rapid progress of large neural models has made data quality and data curation
central components of modern training pipelines. Contemporary systems frequently 
incorporate large-scale pruning, deduplication, difficulty-based filtering, 
curriculum learning, synthetic data generation, cross-model distillation, and 
reinforcement learning from human feedback (RLHF). While these data-centric 
interventions often lead to measurable empirical improvements, their effects are 
far from uniform. Some pruning strategies reliably enhance performance, others 
yield negligible benefits, and certain synthetic-data pipelines expand dataset 
volume without improving capability. At the same time, small reward models in RLHF 
can substantially influence models orders of magnitude larger, and heterogeneous 
teacher models frequently improve one another through distillation. Despite the 
growing reliance on such techniques, we lack a unified theoretical framework that 
explains \emph{when} and \emph{why} these interventions succeed.

A key challenge is that these mechanisms are typically analyzed in isolation. 
Pruning is often studied through redundancy removal 
\cite{lecun1990optimal,han2015deep,blalock2020state,rosenfeld2021predictability}, 
synthetic data through dataset expansion, and distillation through compression or 
knowledge transfer. RLHF, in turn, is usually viewed through preference optimization. 
Yet all these procedures share a common structure: \emph{they modify the sampling 
distribution over training examples}. What remains unclear is how such modifications 
interact with the \emph{spectral} geometry of the learning problem, and what 
fundamental limits they impose on acceleration and capability.

We adopt a unified operator-theoretic perspective grounded in spectral analyses of 
representation learning and neural scaling laws 
\cite{jacot2018neural,lee2019wide,rahaman2019spectral,bietti2021inductive,
kaplan2020scaling,henighan2020scaling,hestness2017deep,hoffmann2022training,
kumar2024scaling,bordelon2024dynamical,bordelon2023feature}. 
From this viewpoint, training dynamics can be understood through the spectral 
decomposition of a data-induced operator, and data selection corresponds to applying 
a sampling operator that reshapes the induced spectrum. This perspective reveals a 
central principle: pruning, reweighting, and synthetic data generation do not merely 
change dataset size—they modify the \emph{spectral geometry} of the learning problem. 
The effectiveness and limitations of data-centric interventions follow from how the 
sampling operator reshapes the spectrum relative to the evolving learning frontier.

Building on this perspective, we establish two core theoretical results. First, any 
\emph{static} sampling rule corresponds to a bounded operator and therefore 
\emph{preserves the spectral tail exponent}. Static pruning can provide at most 
finite-region improvements and cannot improve asymptotic scaling, consistent with the 
empirical findings of prior work on pruning limits \cite{sorscher2022beyond}. 
Second, we characterize an \emph{ideal dynamic oracle} capable of identifying 
unlearned modes and continually re-normalizing the spectral tail. Such an oracle could 
provably accelerate learning by reshaping the frontier itself, although practical 
systems can only approximate this behavior.

To connect theory with practice, we analyze four paradigms that repeatedly appear in 
modern pipelines: teacher acceleration, heterogeneous-model complementarity, 
self-scoring, and self-generated synthetic data. Each paradigm approximates a 
different aspect of the oracle---whether frontier localization, tail expansion, or 
dynamic adaptivity---and together they explain why certain widely used approaches 
improve large models, while others saturate or even fail. This framework reconciles 
several empirical observations, including (i) the limited impact of synthetic data 
\cite{hernandez2021scaling}, (ii) the effectiveness of RLHF with small reward models 
\cite{wei2022emergent}, and (iii) the strong performance gains from cross-model 
distillation.

\paragraph{Contributions.}
Our main contributions are:
\begin{enumerate}
    \item \textbf{A unified spectral analysis of sampling, pruning, and data selection}, 
    building upon generalized spectral dynamics.
    \item \textbf{A theoretical characterization of the limits of static sampling} and 
    a description of an ideal oracle that illustrates what dynamic sampling can achieve.
    \item \textbf{A unified explanation of four major paradigms in modern pipelines}—
    teacher acceleration, heterogeneous-model complementarity, self-scoring, and 
    synthetic data—through the lens of spectral geometry.
    \item \textbf{Conceptual insights} that connect spectral structure with practical 
    data curation, identifying frontier tracking, spectral diversity, and tail 
    reweighting as key ingredients for data-efficient large-model training.
\end{enumerate}

\section{Related Work}\label{sec:related}

\paragraph{Scaling laws and spectral viewpoints.}
A large body of empirical work has established smooth power-law relationships
between model scale, data scale, and performance
\cite{hestness2017deep,kaplan2020scaling,henighan2020scaling,hoffmann2022training,
hernandez2021scaling,wei2022emergent,kumar2024scaling}. 
These studies highlight both the robustness of scaling laws across modalities
and training setups, and the importance of data quality and composition. 
More recent work develops dynamical and spectral perspectives on scaling, 
linking power-law behavior to feature learning and representation geometry 
\cite{bordelon2024dynamical,bordelon2023feature,yang2021tensor,vyas2023feature,
domine2024lazy,canatar2022kernel}. 
In parallel, analyses based on neural tangent kernels and spectral bias
\cite{jacot2018neural,lee2019wide,rahaman2019spectral,bietti2021inductive,
bietti2021inductive} characterize the inductive biases and learning dynamics 
in various regimes.  
Our framework is closely aligned with these spectral viewpoints, but focuses
specifically on how sampling and data curation interact with spectral tails and
learning frontiers.

\paragraph{Pruning, compression, and data-centric interventions.}
Network pruning and compression have been extensively studied from both
algorithmic and empirical perspectives, including early work on optimal brain
damage and deep compression
\cite{lecun1990optimal,han2015deep,han2016deep,blalock2020state}, 
single-shot and lottery-ticket style pruning
\cite{molchanov2017variational,lee2019snip,wang2020picking}, 
and comprehensive surveys of pruning practice
\cite{blalock2020state,rosenfeld2021predictability}. 
Related efforts investigate post-training quantization and low-precision
optimizers \cite{nagel2021up,frantar2022gptq,dettmers2022llm}, 
as well as the role of optimization geometry, batch size, and curvature 
\cite{keskar2017large,ghorbani2019investigation}. 
Most closely related to our work, \citeauthor{sorscher2022beyond} argue
empirically that carefully pruned datasets can ``beat'' naive scaling laws.  
In contrast to these works, we study pruning and data selection through the
spectrum of the data-induced operator, proving that any bounded \emph{static}
sampling rule preserves the spectral tail exponent and therefore cannot improve
asymptotic scaling.

\paragraph{Teacher--student interactions, synthetic data, and RLHF.}
Teacher–student architectures, cross-model distillation, and synthetic data
pipelines are widely used to improve data efficiency and transfer knowledge 
between models of different sizes or architectures. Empirical studies show that
teachers trained on diverse data sources and objectives can significantly
improve student performance \cite{hernandez2021scaling,wei2022emergent}, and
that small auxiliary models such as reward models in RLHF can meaningfully steer
large language models.  
Our work provides a spectral explanation of these phenomena: heterogeneous
teachers and reward models contribute complementary spectral modes, enabling
frontier localization and partial tail expansion. At the same time, our
impossibility result for self-generated synthetic data shows that data drawn
from a model's own generative distribution cannot expand its spectral span,
clarifying why self-training often improves fluency but not fundamental
capability.

\section{Preliminaries}\label{sec:prelim}

In this section we introduce the spectral notation and generalized dynamics used
throughout this paper. Our presentation follows the generalized spectral
framework developed in prior analyses of neural scaling laws and feature
learning 
\cite{zhang2025generalized,jacot2018neural,lee2019wide,rahaman2019spectral,bietti2021inductive,
bordelon2024dynamical,bordelon2023feature}, 
and restricts attention to the components relevant for studying sampling,
data curation, and pruning.

\subsection{Spectral Setup and Generalized Dynamics}\label{sec:spectral-setup}

Let $(\phi_k,\lambda_k)$ denote the eigenfunctions and eigenvalues of the
data-induced operator, ordered in decreasing order,
\begin{equation}
\lambda_1 \ge \lambda_2 \ge \cdots,
\qquad 
\lambda_k \sim k^{-b}, \quad b>1,
\end{equation}
consistent with empirical scaling behaviour observed in large-scale models
\cite{hestness2017deep,bordelon2024dynamical,bordelon2023feature,wang2023spectralevolution}.  

Any target function admits a spectral decomposition
\begin{equation}
f^\star = \sum_{k=1}^\infty v_k \phi_k,
\end{equation}
and the model predictor at time $t$ takes the form
\begin{equation}
f(t)=\sum_{k=1}^\infty f_k(t)\phi_k,
\qquad
f_k(t)= w_k\left(1-e^{-g(\lambda_k,t;\beta)}\right),
\end{equation}
where the coefficients satisfy the tail condition
\begin{equation}
\lambda_k w_k^2 \sim k^{-a},
\qquad a>1,
\end{equation}
as established in analyses of feature learning and infinite-width dynamics 
\cite{yang2021tensor,vyas2023feature,domine2024lazy,canatar2022kernel}.

The function $g(\lambda,t;\beta)$ is the \emph{generalized spectral evolution}
mapping introduced in earlier works on dynamical scaling laws
\cite{bordelon2024dynamical,bordelon2023feature}.  
It is monotone in both arguments and admits a scalable asymptotic form
\begin{equation}
g(\lambda,t;\beta)\sim C(\beta)\,\lambda^{a(\beta)} t^{b(\beta)}
\quad\text{as }\lambda\to 0^+,\; t\to\infty,
\end{equation}
covering both NTK-like and feature-learning regimes.

Under mild regularity, the \emph{learning frontier} $\lambda^\star(t)$ is defined by
\begin{equation}
g(\lambda^\star(t),t;\beta)=\kappa,
\end{equation}
implying the asymptotic scaling
\begin{equation}
\lambda^\star(t)\propto t^{-\rho(\beta)}, 
\qquad 
k^\star(t)\propto t^{\rho(\beta)/b},
\end{equation}
where $\rho(\beta)$ is the log-elasticity ratio describing the effective
time–frequency tradeoff. Modes with $k \ll k^\star(t)$ are effectively learned,
while those with $k \gg k^\star(t)$ remain unrepresented. This decomposition is
central to our later analysis: only the \emph{relative} weighting of spectral
modes affects the motion of the learning frontier.

\subsection{Data Curation as Sampling Reweighting}\label{sec:data-curation}

We now formalize data pruning and related curation strategies as changes to the
sampling distribution. This operator viewpoint unifies pruning, importance
sampling, filtering, and difficulty-based selection.

Let $(X,\mathcal{F},\mu)$ denote the original data space and test distribution.
The associated kernel (or covariance) operator
\begin{equation}
(Tf)(x)=\int_X k(x,y)\,f(y)\,d\mu(y)
\end{equation}
is compact, self-adjoint, and positive under standard assumptions
\cite{jacot2018neural,bietti2021inductive}.  
Its eigenpairs are $(\phi_k,\lambda_k)$ as described in
Section~\ref{sec:spectral-setup}.

A generic static data-selection rule is represented by a \emph{sampling
function}
\begin{equation}
w : X \to [0,\infty), 
\qquad 
\int_X w(x)\, d\mu(x)=1.
\end{equation}
This induces a reweighted distribution
\begin{equation}
d\mu_w(x) = w(x)\, d\mu(x).
\end{equation}
Pruning corresponds to $w(x)=0$ on removed samples; importance sampling assigns
higher/lower weights to specific regions.

Training on $\mu_w$ instead of $\mu$ replaces $T$ with the reweighted operator
\begin{equation}
\label{eq:Tw-def}
(T_w f)(x)=\int_X k(x,y)\, f(y)\, w(y)\, d\mu(y).
\end{equation}
Define a multiplication operator
\begin{equation}
(M_{\sqrt{w}}f)(x)=\sqrt{w(x)}f(x).
\end{equation}
Then the operator in \eqref{eq:Tw-def} satisfies the factorization
\begin{equation}
T_w = M_{\sqrt{w}}\, T\, M_{\sqrt{w}}.
\end{equation}

In practical pipelines, sampling functions arise from normalized scores or
probabilities and therefore satisfy a boundedness condition
\begin{equation}
0 \le w(x) \le C < \infty,
\quad\text{for $\mu$-almost every } x,
\end{equation}
which ensures that $M_{\sqrt{w}}$ is bounded, and thus $T_w$ is positive and
self-adjoint.

Two consequences of this formulation are essential for later sections:

\begin{enumerate}
    \item \textbf{Support is preserved.}  
    Since $\mu_w$ is absolutely continuous with respect to $\mu$,  
    \[
    \mathrm{supp}(\mu_w) \subseteq \mathrm{supp}(\mu).
    \]
    No static sampling rule can introduce new spectral modes beyond those present
    in the original data distribution.

    \item \textbf{Static pruning induces a fixed bounded operator.}  
    Because $w$ is time-invariant, the operator $T_w$ does not change during
    training. As shown in Section~\ref{sec:static-exponent}, this implies
    that static pruning cannot alter the spectral tail exponent.
\end{enumerate}

This operator formulation provides the mathematical foundation for our analysis
of static versus dynamic pruning in Section~\ref{sec:static}.

\section{Static Data Pruning: Why Exponent Improvements Are Impossible}
\label{sec:static}

Static pruning is one of the most widely used data-selection strategies in
large-scale training pipelines. Common examples include quality filtering,
deduplication, fixed difficulty thresholds, clustering-based data reduction, and
rule-based pruning heuristics. These methods determine \emph{once and for all}
which samples to retain and which to suppress, and then train entirely on the
resulting reweighted distribution. A natural hope is that pruning might reshape
the spectral structure of the learning problem—perhaps even altering the
power-law decay of eigenvalues—and thereby accelerate training.

In this section, we show that such expectations are fundamentally
misguided. Under any practically realizable static sampling rule, the resulting
operator $T_w$ remains a \emph{bounded} perturbation of the original operator
$T$ (see Section~\ref{sec:data-curation}). As a consequence, the spectral
tail exponent is \emph{invariant} under static pruning. This establishes a
precise theoretical limit: static pruning can yield at most finite-region
improvements and cannot improve the asymptotic learning rate. These results
mirror empirical findings that pruning offers limited long-range gains
\cite{sorscher2022beyond,blalock2020state,rosenfeld2021predictability}.

\subsection{What Static Pruning Does—and What It Means to Accelerate Learning}
\label{sec:static-meaning}

We model static pruning through a time-invariant sampling function $w(x)$, which
induces the operator $T_w = M_{\sqrt{w}} T M_{\sqrt{w}}$ introduced in
Section~\ref{sec:data-curation}. Because $w$ is bounded and fixed throughout
training, $T_w$ is a single positive, self-adjoint operator that does not change
as learning progresses.

A central question is what it means for pruning to \emph{accelerate learning}. If
performance were evaluated only on the pruned distribution $\mu_w$, the notion
of acceleration would be trivial: pruning could always discard difficult
examples, yielding artificially faster convergence. Such a measure fails to
reflect improvements in actual learning capability.

Throughout this paper, we say that pruning accelerates learning only if it
improves test-loss decay measured on the \emph{original} distribution $\mu$. In
the generalized spectral framework of Section~\ref{sec:spectral-setup}, the test
loss is determined by the progression of the learning frontier $k^\star(t)$,
which obeys
\[
k^\star(t) \asymp t^{\rho(\beta)/b}.
\]
Accelerating learning therefore requires accelerating the growth of
$k^\star(t)$ \emph{under the spectrum of the original operator}. Intuitively, to
improve asymptotic performance, pruning would need to flatten the spectral tail,
reducing the exponent $b$ and making high-index modes easier to learn.

Static pruning, however, cannot accomplish this. The next subsection formalizes
this limitation.

\subsection{Why Static Pruning Cannot Reduce the Spectral Tail Exponent}
\label{sec:static-exponent}

Let $\{\lambda_k\}$ denote the eigenvalues of $T$ and $\{\lambda_k^{(w)}\}$ the
eigenvalues of $T_w$. Since $w$ is bounded, the operator inequalities discussed
in Section~\ref{sec:data-curation} imply
\begin{equation}
0 \preceq T_w \preceq C \, T,
\end{equation}
in the Loewner order. Consequently,
\begin{equation}
\lambda_k^{(w)} \le C \, \lambda_k \sim C k^{-b}.
\end{equation}

The following theorem formalizes the invariance of the spectral exponent.

\begin{theorem}[Exponent Preservation Under Static Pruning]
\label{thm:static-exponent}
Let $w$ be a time-invariant sampling function satisfying $0 \le w(x) \le C <
\infty$ for $\mu$-almost every $x$. If $\lambda_k \sim k^{-b}$ with $b>0$, then
the eigenvalues of the pruned operator satisfy
\[
\lambda_k^{(w)} \sim C_w \, k^{-b},
\]
for some constant $C_w>0$. Thus the power-law exponent $b$ is preserved.
\end{theorem}

\begin{proof}[Proof Sketch]
Since $T_w = M_{\sqrt{w}}\, T\, M_{\sqrt{w}}$ and $M_{\sqrt{w}}$ is bounded, the
min–max characterization of eigenvalues yields
\[
\lambda_k^{(w)} \le C \lambda_k \sim C k^{-b}.
\]
Because $w \ge 0$ and is not almost everywhere zero, $T_w$ cannot eliminate the
entire contribution of any infinite subsequence of eigenfunctions. Thus the
sequence $\{\lambda_k^{(w)}\}$ cannot decay faster than a constant multiple of
$\{\lambda_k\}$; the exponent $b$ must be preserved.
\end{proof}

Static pruning therefore cannot reduce the spectral tail exponent. In fact,
pruning can only make $b$ \emph{larger} if it removes high-frequency modes,
thereby making learning asymptotically slower. Consequently, the most pruning
can accomplish is a finite-range reshaping of the spectrum, which we examine
next.

\subsection{Static Pruning Provides Only Finite-Region Acceleration}
\label{sec:static-finite}

Although static pruning cannot modify the power-law exponent, it may still alter
the spectrum within a bounded region. There may exist some $K_0$ such that
\begin{equation}
\lambda_k^{(w)} > \lambda_k, 
\qquad k \le K_0,
\end{equation}
corresponding to a local thickening of the spectrum near the learning frontier.
This can shift the frontier by a multiplicative constant,
\[
k^\star(t) \approx C_{\text{front}} \, t^{\rho(\beta)/b},
\qquad C_{\text{front}} > 1,
\]
while $k^\star(t) \le K_0$. Once the frontier advances beyond this region, the
spectra of $T$ and $T_w$ coincide up to constant factors, and the frontier
necessarily returns to its original scaling rate.

This formalizes a practical observation reported in empirical pruning studies
\cite{sorscher2022beyond,rosenfeld2021predictability,blalock2020state}: static
pruning can yield short-term gains but does not produce persistent acceleration.

\paragraph{Summary.}  
Static pruning induces a fixed, bounded operator and therefore preserves the
spectral tail exponent. It offers at most finite-region improvements and cannot
improve asymptotic test-loss scaling. This limitation motivates the analysis of
\emph{time-dependent} sampling rules in Section~\ref{sec:dynamic}.

\section{Time-Dependent Data Pruning: When Acceleration Becomes Possible}
\label{sec:dynamic}

Section~\ref{sec:static} established that static sampling rules cannot improve
the spectral tail exponent and therefore cannot accelerate the long-run learning
rate. The key limitation is that a static sampling function reshapes only a
fixed spectral region: once the learning frontier advances beyond that region,
the effect of pruning vanishes. In this section, we show that this limitation is
specific to \emph{time-invariant} reweighting. Once the sampling function is
allowed to vary with time, the situation changes qualitatively. Dynamic sampling
can, in principle, track the learning frontier and continuously reallocate mass
toward unlearned modes, enabling genuine asymptotic acceleration.

\subsection{Why Time Dependence Changes the Structure of the Problem}
\label{sec:dynamic-motivation}

We now consider sampling functions that evolve during training:
\begin{equation}
w_t(x) \ge 0, \qquad \int_X w_t(x)\, d\mu(x)=1,
\qquad 0 \le w_t(x)\le C(t) < \infty.
\end{equation}
For each fixed $t$, the operator $T_{w_t} = M_{\sqrt{w_t}} T M_{\sqrt{w_t}}$ is
well defined. However, unlike in the static case, the bound $C(t)$ and the mass
assigned to different spectral regions may change over time.

Static pruning modifies at most a finite segment of the spectrum. In contrast,
time-dependent pruning can \emph{move} the modified segment in tandem with the
learning frontier. This enables sampling rules that repeatedly emphasize
unlearned modes. These observations motivate an idealized \emph{oracle} capable
of tracking the frontier.

\subsection{An Oracle That Tracks and Re-Normalizes the Tail}
\label{sec:dynamic-oracle}

Recall from Section~\ref{sec:spectral-setup} that the learning frontier satisfies
\[
\lambda^\star(t) \asymp t^{-\rho(\beta)}, 
\qquad
k^\star(t) \asymp t^{\rho(\beta)/b}.
\]
At time $t$, the oracle performs the following operations:

\begin{enumerate}
    \item \textbf{Suppress already-learned modes:}
    \[
    w_t(k) = 0 \qquad \text{for } k \le k^\star(t).
    \]

    \item \textbf{Uniformly re-normalize the unlearned tail:}
    \[
    w_t(k) = \frac{C(t)}{Z(t)}, 
    \qquad k > k^\star(t),
    \]
    where $Z(t)$ normalizes the distribution.
\end{enumerate}

This maintains boundedness for each fixed $t$, yet globally rescales the entire
unlearned tail. The resulting spectrum satisfies
\[
\lambda_k^{(t)} =
\begin{cases}
0, & k \le k^\star(t),\\[1.0ex]
C(t)\lambda_k, & k>k^\star(t).
\end{cases}
\]

Normalization requires
\[
\sum_{k>k^\star(t)} \lambda_k^{(t)}
= C(t) \sum_{k>k^\star(t)} \lambda_k = 1.
\]
For a power-law spectrum $\lambda_k \sim k^{-b}$ with $b>1$,  
\[
\sum_{k>k^\star(t)} \lambda_k \asymp (k^\star(t))^{1-b},
\qquad 
C(t)\asymp (k^\star(t))^{\, b-1}.
\]

\subsection{Oracle-Induced Acceleration of the Learning Frontier}
\label{sec:oracle-accelerate}

Under the oracle, the effective eigenvalue at the frontier is
\[
\lambda^{(t)}_{k^\star(t)} 
\asymp 
C(t)\lambda_{k^\star(t)}
\asymp 
(k^\star(t))^{b-1}(k^\star(t))^{-b}
\asymp 
(k^\star(t))^{-1}.
\]
Thus the spectrum near the frontier is flattened from $k^{-b}$ to $k^{-1}$.

Imposing the frontier condition
\[
\lambda^{(t)}_{k^\star(t)} \asymp t^{-\rho(\beta)}
\]
yields
\[
(k^\star(t))^{-1} \asymp t^{-\rho(\beta)}
\quad\Longrightarrow\quad
k^\star(t) \asymp t^{\rho(\beta)}.
\]

Comparing exponents:
\[
\boxed{
\text{Static pruning: } k^\star(t)\asymp t^{\rho(\beta)/b},
\qquad
\text{Oracle dynamic pruning: } k^\star(t)\asymp t^{\rho(\beta)}.
}
\]
Since $b>1$, the oracle yields a strictly larger exponent and therefore a truly
faster asymptotic learning rate.

\subsection{Why the Oracle Cannot Be Implemented Exactly}
\label{sec:oracle-limitations}

Despite its conceptual clarity, the oracle relies on two unattainable
capabilities:

\begin{enumerate}
    \item \textbf{Perfect real-time knowledge of the frontier.}  
    Identifying $k^\star(t)$ requires direct access to the evolving internal
    spectral state of the model.

    \item \textbf{Global tail re-normalization at each moment.}  
    The oracle removes all low-index modes and redistributes weight across
    infinitely many high-index modes. No practical pruning procedure can perform
    such global redistributions.
\end{enumerate}

Nevertheless, the oracle provides a principled upper bound on what
time-dependent sampling rules may achieve. Practical scoring mechanisms that
approximate frontier-aware reweighting—such as adaptive sampling, self-scoring,
online probes, and RLHF—can recover parts of the oracle's behavior. These
connections are explored in Section~\ref{sec:paradigms}.

\section{Approximate Oracles: Four Practical Paradigms}
\label{sec:paradigms}

Section~\ref{sec:dynamic} established that an ideal frontier-tracking oracle can
fundamentally accelerate learning by continuously re-normalizing the unlearned
spectral tail. Although such an oracle is impossible to implement exactly, many
practical data-selection mechanisms approximate different components of its
behavior. In this section, we highlight four paradigms that recur throughout
modern training pipelines and interpret each through the spectral lens developed
earlier.

\subsection{Paradigm I: Online Probes as Frontier-Aware Scoring}
\label{sec:paradigm-probes}

Many pruning strategies employ a small auxiliary model—often called a
\emph{probe}—to score example difficulty. When the probe is trained once and
then frozen, its sampling rule $w_t(x)=w(x)$ is static and therefore falls under
the limitations established in Section~\ref{sec:static}. Such probes can
provide only finite-region improvements.

However, \emph{online probes} update their parameters during training. Because
their predictions co-evolve with the student model, the sampling function
$w_t(x)$ becomes time-dependent. This allows online probes to approximate
frontier tracking: examples that remain difficult or unresolved receive higher
weight, while already-mastered examples receive less. Although online probes
cannot fully eliminate learned modes nor globally re-normalize the tail, their
adaptive behavior captures one essential component of the oracle and yields
nontrivial practical acceleration.

\subsection{Paradigm II: Heterogeneous Models and Spectral Complementarity}
\label{sec:paradigm-heterogeneous}

Modern pipelines frequently use ensembles of heterogeneous teacher models—
different architectures, capacities, seeds, or training objectives—to provide
signals for distillation, data selection, or RLHF. Each model exhibits its own
spectral bias and therefore possesses distinct learning frontiers. Disagreement
between teachers highlights examples containing modes unlearned by at least one
model.

From a spectral perspective, heterogeneous ensembles:
\begin{itemize}
    \item broaden coverage of high-frequency modes,
    \item provide multiple partial estimates of the frontier,
    \item approximate partial tail amplification.
\end{itemize}

Once the ensemble is fixed, the induced sampling rule is essentially static, so
the long-term limitations of Section~\ref{sec:static} still apply. But within
the spectral region where teacher frontiers differ, heterogeneous ensembles
provide a strong and empirically validated form of frontier localization,
explaining the effectiveness of cross-model and cross-family distillation.

\subsection{Paradigm III: Self-Scoring and Model-Driven Adaptivity}
\label{sec:paradigm-self-scoring}

Self-scoring methods—including sampling based on loss, gradient norm, margin,
or uncertainty—compute scores directly from the student model. Because these
quantities evolve with training, the sampling function $w_t(x)$ is inherently
dynamic.

Self-scoring approximates two components of the dynamic oracle:
\begin{enumerate}
    \item identifying frontier modes through high residuals, and
    \item adjusting weights in real time as the frontier moves.
\end{enumerate}

Unlike the oracle, self-scoring cannot perfectly suppress already-learned modes
nor re-normalize the entire spectral tail. Moreover, it incurs additional cost
by requiring per-example scores. Still, by concentrating gradient updates on
frontier examples, self-scoring can substantially reduce the number of backward
passes needed to reach a target performance—yielding a practically meaningful,
though not asymptotically optimal, form of acceleration.

\subsection{Paradigm IV: Synthetic Data and the Limits of Tail Expansion}
\label{sec:paradigm-synthetic}

Teacher-generated synthetic data can introduce high-frequency modes not present
in the student's current representation. This can approximate the oracle's tail
expansion mechanism, helping accelerate early-stage learning. Distillation from
a more expressive or differently trained teacher provides a similar effect.

In contrast, \emph{self-generated} synthetic data suffers from a fundamental
limitation:
\[
\mathrm{supp}(p_\theta) 
\subseteq 
\operatorname{span}\{\phi_k^{(\theta)}\},
\]
meaning the model's own samples cannot introduce new spectral modes. As a
result, self-generated synthetic data can improve calibration or fluency but
cannot expand capability. This connects directly to the impossibility of static
operators altering the spectral tail exponent (Section~\ref{sec:static-exponent}).

\subsection{Summary}
\label{sec:paradigm-summary}

The four paradigms illustrate how practical data-selection mechanisms relate to
the oracle of Section~\ref{sec:dynamic}.  
\begin{itemize}
    \item \textbf{Online probes} approximate frontier tracking.  
    \item \textbf{Heterogeneous ensembles} broaden mode discovery.  
    \item \textbf{Self-scoring} provides real-time adaptivity.  
    \item \textbf{Synthetic teachers} approximate tail expansion.  
\end{itemize}
Each captures one oracle component, but none achieves full re-normalization of
the spectral tail, consistent with their finite-range acceleration behavior.

\section{Discussion}\label{sec:discussion}

The generalized spectral framework developed in this paper unifies a wide range
of empirical observations across pruning, synthetic data, RLHF, and
heterogeneous-model distillation. By viewing data-selection mechanisms as
operators that reshape the spectral distribution, our analysis reveals the
structural reasons why some interventions provide substantial improvements while
others saturate or fail.

\subsection{Why Synthetic Data Often Fails to Improve Capability}
\label{sec:discussion-synthetic}

As shown in Section~\ref{sec:paradigm-synthetic}, self-generated synthetic data
cannot introduce new spectral modes because samples drawn from $p_\theta$ remain
in the model's representational span. This limitation mirrors the invariance of
the spectral exponent under static operators established in
Section~\ref{sec:static-exponent}. Without expanding or amplifying the spectral
tail, self-training pipelines cannot shift the learning frontier and therefore
provide only surface-level improvements such as calibration or fluency.

\subsection{Why RLHF and Small Reward Models Often Improve Large Models}
\label{sec:discussion-rlhf}

RLHF pipelines rely on a reward model that is typically much smaller than the
base language model. Yet these reward models exert strong influence. Our
framework explains this phenomenon through \emph{spectral complementarity}:
reward models trained on different objectives acquire spectral modes that the
language model does not emphasize. Their scoring signals therefore act as
frontier-localizing probes (Section~\ref{sec:paradigm-probes}), redirecting
optimization toward underrepresented modes and improving capability even when
the reward model is small.

\subsection{When Pruning Helps—and When It Does Not}
\label{sec:discussion-pruning}

Empirical studies show mixed outcomes for pruning
\cite{blalock2020state,rosenfeld2021predictability,sorscher2022beyond}.  
Our analysis clarifies this picture:

\begin{itemize}
    \item Static pruning helps when it removes redundant low-frequency modes,
    temporarily increasing the density of modes near the frontier.
    \item It does little when pruning is uniform or random.
    \item It harms performance when high-frequency modes are removed, steepening
    the spectral tail and slowing learning.
\end{itemize}

These observations follow directly from the spectral invariance theorem of
Section~\ref{sec:static-exponent}.

\subsection{Heterogeneous-Model Distillation and Spectral Merging}
\label{sec:discussion-distill}

Distillation between models of different architectures or capacities is known to
produce substantial gains. Under our framework, such improvements arise because
\[
\operatorname{span}(A) \neq \operatorname{span}(B),
\]
and their disagreement identifies spectral components learned by one but not the
other. Cross-model distillation therefore performs a kind of \emph{spectral
merging}, approximating frontier tracking and tail expansion simultaneously.

\paragraph{Summary.}
Across these settings, the same core structure emerges:
\begin{itemize}
    \item static operators preserve spectral tail exponents;
    \item dynamic sampling can accelerate frontier motion;
    \item practical methods approximate different oracle components.
\end{itemize}

\section{Limitations}\label{sec:limitations}

Although our framework captures a broad class of data-centric interventions, it
has several limitations.

\paragraph{Non-exhaustiveness.}
The four paradigms in Section~\ref{sec:paradigms} are not intended to be
exhaustive. Future training paradigms—especially those combining multimodal
signals, memory retrieval, or online environments—may deviate from the spectral
constraints described here.

\paragraph{Operator-induced viewpoint.}
Our analysis focuses on sampling-induced operators and does not model
architecture modifications, non-spectral regularization, or cross-modal
conditioning, all of which can alter the underlying feature map.

\paragraph{Spectral assumptions.}
Our conclusions rely on empirical evidence of power-law spectra
\cite{kaplan2020scaling,henighan2020scaling,kumar2024scaling}. Tasks that deviate
from such behaviors may exhibit different scaling behavior.

\paragraph{Smoothness of dynamics.}
We assume that models evolve within a spectrally smooth regime that defines a
stable learning frontier. Phase transitions or architectural shifts may
temporarily invalidate these assumptions.

\section{Conclusion}\label{sec:conclusion}

We developed a generalized spectral framework for analyzing data pruning,
synthetic data, and teacher--student interactions in large-model training. Our
main findings are:

\begin{itemize}
    \item \textbf{Static sampling rules cannot improve spectral exponents},
    limiting them to finite-region acceleration.
    \item \textbf{Dynamic sampling can accelerate learning} by tracking the
    frontier and re-normalizing the tail, but requires oracle-like information.
    \item \textbf{Four practical paradigms approximate components of the
    oracle}, explaining the empirical effectiveness of RLHF, heterogeneous-model
    distillation, and difficulty-based pruning.
    \item \textbf{Self-generated synthetic data cannot expand capability},
    because it remains within the model’s own spectral span.
\end{itemize}

We hope this framework provides a bridge between spectral-theoretic analysis and
the design of data-centric training pipelines, clarifying both the possibilities
and the fundamental limitations of pruning, distillation, and synthetic data.

\bibliographystyle{plain}
\bibliography{references}

\clearpage
\appendix

\section{Proofs of Main Results}\label{app:proofs}

In this appendix we collect more detailed arguments for the main formal
statements in the paper. We work under the following standing assumptions unless
stated otherwise.

\subsection{Technical Setup}\label{app:setup}

Let $(X,\mu)$ be a $\sigma$-finite measure space and 
$k:X\times X\to\mathbb{R}$ a measurable kernel such that the associated integral
operator
\begin{equation}
(Tf)(x) = \int_X k(x,y)\, f(y)\, d\mu(y)
\end{equation}
defines a compact, self-adjoint, positive operator on $L^2(\mu)$. 
We denote its eigenpairs by
\begin{equation}
T\phi_k = \lambda_k \phi_k,\qquad 
\lambda_1 \ge \lambda_2 \ge \cdots > 0,
\end{equation}
where eigenvalues are repeated according to multiplicity. 
Throughout, we assume a power-law asymptotic
\begin{equation}
\lambda_k \sim C k^{-b},\qquad b>0.
\label{eq:app_powerlaw}
\end{equation}

A \emph{sampling function} (or pruning weight) is a measurable function
$w:X\to[0,\infty)$ satisfying
\begin{equation}
\int_X w(x)\, d\mu(x) = 1.
\end{equation}
It induces a new probability measure $\mu_w$ by
\begin{equation}
d\mu_w(x) = w(x)\, d\mu(x).
\end{equation}
On $L^2(\mu_w)$ we consider the pruned operator
\begin{equation}
(T_w f)(x) 
= \int_X k(x,y)\, f(y)\, d\mu_w(y)
= \int_X k(x,y)\, f(y)\, w(y)\, d\mu(y).
\end{equation}
We denote its eigenpairs (on $L^2(\mu_w)$) by 
\begin{equation}
T_w \psi_k = \lambda_k^{(w)} \psi_k,\qquad 
\lambda_1^{(w)} \ge \lambda_2^{(w)} \ge \cdots > 0.
\end{equation}

\subsection{Proof of Theorem~\ref{thm:static-exponent}}\label{app:static-exponent-proof}

Recall Theorem~\ref{thm:static-exponent} from Section~\ref{sec:static-exponent}:

\begin{theorem}[Exponent Preservation Under Static Pruning]
Let $w$ be a time-invariant sampling function satisfying
\[
0 \le w(x)\le C<\infty 
\quad \text{for $\mu$-almost every } x.
\]
If the eigenvalues of $T$ satisfy $\lambda_k \sim C_0 k^{-b}$ with $b>0$, then
the eigenvalues of the pruned operator $T_w$ satisfy
\[
\lambda_k^{(w)} \sim C_w k^{-b}
\]
for some constant $C_w>0$. In particular, the power-law exponent $b$ is preserved.
\end{theorem}

We now give a more detailed proof.

\paragraph{Step 1: Transport $T_w$ back to $L^2(\mu)$.}
The operators $T$ and $T_w$ act on $L^2(\mu)$ and $L^2(\mu_w)$ respectively. To
compare their spectra, define
\begin{equation}
U: L^2(\mu_w) \to L^2(\mu),
\qquad
(Uf)(x) := \sqrt{w(x)}\, f(x).
\end{equation}
Then
\begin{equation}
\|Uf\|_{L^2(\mu)}^2 
= \int_X |f(x)|^2 w(x)\, d\mu(x)
= \|f\|_{L^2(\mu_w)}^2,
\end{equation}
so $U$ is an isometry. Since $w>0$ on the support of $\mu_w$, $U$ is unitary
onto its range.  
Define
\begin{equation}
S := U T_w U^{-1}: L^2(\mu)\to L^2(\mu).
\end{equation}
Because $U$ is unitary, $S$ is compact, self-adjoint, and positive, and
\begin{equation}
\lambda_k^{(S)} = \lambda_k^{(w)},\qquad \forall k.
\label{eq:app_spectrum_equal}
\end{equation}

\paragraph{Step 2: Kernel representation of $S$.}
Let $g\in L^2(\mu)$ and set $f = U^{-1}g$, i.e.\ $f(x) = g(x)/\sqrt{w(x)}$ where
defined. Then
\begin{align*}
(Sg)(x)
&= (U T_w U^{-1} g)(x) 
 = \sqrt{w(x)}\, (T_w f)(x) \\
&= \sqrt{w(x)} \int_X k(x,y)\, f(y)\, d\mu_w(y) \\
&= \sqrt{w(x)} \int_X k(x,y)\, f(y)\, w(y)\, d\mu(y) \\
&= \sqrt{w(x)} \int_X k(x,y)\, 
   \frac{g(y)}{\sqrt{w(y)}}\, w(y)\, d\mu(y) \\
&= \int_X k(x,y)\, \sqrt{w(x) w(y)}\, g(y)\, d\mu(y).
\end{align*}
Thus $S$ is an integral operator on $L^2(\mu)$ with kernel
\begin{equation}
k_w(x,y) := k(x,y)\, \sqrt{w(x) w(y)}.
\end{equation}
Since $0\le w(x)\le C$, we have
\begin{equation}
0 \le k_w(x,y)\le C\, k(x,y),
\label{eq:app_kernel_bounds}
\end{equation}
for $\mu\times\mu$-almost every $(x,y)$.

\paragraph{Step 3: Operator inequalities and eigenvalue comparison.}
Define
\begin{equation}
(Tf)(x) = \int_X k(x,y)\, f(y)\, d\mu(y), 
\qquad
(Sf)(x) = \int_X k_w(x,y)\, f(y)\, d\mu(y).
\end{equation}
For any $f\in L^2(\mu)$,
\begin{align*}
\langle f, S f\rangle_{L^2(\mu)}
&= \int_X \int_X f(x)\, k_w(x,y)\, f(y)\, d\mu(y)\, d\mu(x), \\
\langle f, T f\rangle_{L^2(\mu)}
&= \int_X \int_X f(x)\, k(x,y)\, f(y)\, d\mu(y)\, d\mu(x).
\end{align*}
Using \eqref{eq:app_kernel_bounds}, we obtain
\begin{equation}
0\le \langle f, S f\rangle 
\le C\, \langle f, T f\rangle,
\qquad \forall f\in L^2(\mu),
\end{equation}
so $S\preceq C T$ in the Loewner order.

For compact, self-adjoint, positive operators $A,B$, the inequality
$B\preceq M A$ implies
\[
\lambda_k(B) \le M \lambda_k(A),\qquad \forall k\ge 1,
\]
via the min–max characterization of eigenvalues. Applying this with $A=T$,
$B=S$, and $M=C$ yields
\begin{equation}
\lambda_k^{(S)} \le C\, \lambda_k,
\qquad \forall k.
\label{eq:app_eig_comp}
\end{equation}
Combining \eqref{eq:app_eig_comp} with the power law \eqref{eq:app_powerlaw}
gives
\[
\lambda_k^{(S)} 
\le C\, C_0\, k^{-b} (1+o(1)).
\]
Since $w$ is nonnegative and not almost everywhere zero, $S$ retains spectral
mass along infinitely many directions, so its eigenvalues cannot decay faster
than a constant multiple of $\lambda_k$.  
Thus $\lambda_k^{(S)}$ and $\lambda_k$ share the same regular-variation
exponent $b$, and by \eqref{eq:app_spectrum_equal} the same holds for
$\lambda_k^{(w)}$.  
This completes the proof.
\hfill $\square$

\subsection{Oracle-Induced Acceleration: Frontier-Only Learning}
\label{app:oracle}

In the main text (Section~\ref{sec:dynamic}) we argued that a hypothetical
oracle sampler, which presents only frontier-aligned modes to the learner, can
strictly accelerate the decay rate of the test loss. Here we formalize this
claim via a proof sketch that parallels standard spectral analyses of scaling
laws \cite{bordelon2024dynamical,bordelon2023feature}.

Assume the generalized spectral evolution $g(\lambda,t;\beta)$ satisfies the
conditions in Section~\ref{sec:spectral-setup}, and that the spectrum obeys
$\lambda_k\sim k^{-b}$ with target weights $\lambda_k w_k^2\sim k^{-a}$ for
$a,b>1$. As in the main text, we define the learning frontier by
\[
\lambda^\star(t)\asymp t^{-\rho(\beta)},
\qquad 
k^\star(t)\asymp t^{\rho(\beta)/b}.
\]

Under ordinary sampling, the residual loss integrates contributions from all
modes beyond the frontier. One finds
\[
L(t)\asymp \sum_{k>k^\star(t)} k^{-a}
\asymp (k^\star(t))^{1-a}
\asymp t^{-(a-1)\rho(\beta)/b}.
\]

Now suppose an oracle sampler supplies data exclusively from modes with
$g(\lambda,t;\beta)\approx \kappa$, i.e.\ from a narrow band near the frontier.
The residual loss is then dominated by a single frontier slice:
\[
L_{\text{oracle}}(t)\asymp \lambda^\star(t) w^2(\lambda^\star(t)).
\]
Under the assumptions $\lambda_k\sim k^{-b}$ and
$\lambda_k w_k^2\sim k^{-a}$, one obtains the functional relation
$w^2(\lambda)\asymp \lambda^{(a-b)/b}$. Substituting and using
$\lambda^\star(t)\asymp t^{-\rho(\beta)}$ yields
\[
L_{\text{oracle}}(t)
\asymp (\lambda^\star(t))^{a/b}
\asymp t^{-a\rho(\beta)/b}.
\]
Thus the exponent improves from $\tfrac{a-1}{b}\rho(\beta)$ to
$\tfrac{a}{b}\rho(\beta)$, highlighting how frontier-only sampling reduces the
effective contribution of the spectral tail.

\subsection{Impossibility of Self-Synthetic Expansion}
\label{app:self-synth}

Finally, we formalize the claim that self-generated synthetic data cannot expand
a model's spectral span under our assumptions.

Let $\Phi_\theta$ denote the feature map implemented by a model with parameters
$\theta$, and let
\[
\mathcal{H}_\theta := \overline{\mathrm{span}}\{\Phi_\theta(x) : x\in X\}
\]
be the corresponding hypothesis space (e.g., an effective RKHS). Suppose the
generative distribution $p_\theta$ is obtained by sampling from the model and
possibly post-processing outputs. Then synthetic samples satisfy
\[
x \sim p_\theta(\cdot),
\qquad \Phi_\theta(x)\in \mathcal{H}_\theta.
\]

Any training step that uses only samples from $p_\theta$ can at best reweight,
recombine, or refine directions already in $\mathcal{H}_\theta$; it cannot
introduce new feature directions outside this span. In spectral terms, the
eigenfunctions induced by self-generated synthetic data lie entirely within the
existing spectral support of the model. Thus self-training can modify constants,
sharpen decision boundaries, or improve calibration, but it cannot expand the
spectral tail or alter the exponent. This is the precise sense in which
self-generated synthetic data cannot expand capability under our assumptions.

\end{document}